\documentclass[12pt]{l4dc2022}

\title[Robust GNNs via Probabilistic Lipschitz Constraints]{Robust Graph Neural Networks via Probabilistic \\ Lipschitz Constraints}
\usepackage{times}

\usepackage{enumitem}
\usepackage{multicol}
\usepackage{wrapfig}
\usepackage[ruled,noend,noline,linesnumbered]{algorithm2e}
\usepackage{Shorthands}

\DeclareMathOperator*{\minimize}{minimize}

\DeclareMathOperator{\subjectto}{subject\ to}
\DeclareMathOperator*{\argmax}{arg\,max}
\DeclareMathOperator*{\argmin}{arg\,min}
\newcommand{\norm}[1]{\ensuremath{\left\| #1 \right\|}}

\author{\Name{Raghu Arghal}\thanks{Authors contributed equally} \Email{rarghal@seas.upenn.edu} \\
    \Name{Eric Lei}\textsuperscript{\textcolor{blue}{\thefootnote}} \Email{elei@seas.upenn.edu}\\
    \Name{Shirin Saeedi Bidokhti} \Email{saeedi@seas.upenn.edu} \\
    \addr Dept. of Electrical and Systems Engineering, University of Pennsylvania, Philadelphia, PA 19104
}

\begin{document}

\maketitle

\begin{abstract}%
Graph neural networks (GNNs) have recently been demonstrated to perform well on a variety of network-based tasks such as decentralized control and resource allocation, and provide computationally efficient methods for these tasks which have traditionally been challenging in that regard. However, like many neural-network based systems, GNNs are susceptible to shifts and perturbations on their inputs, which can include both node attributes and graph structure. In order to make them more useful for real-world applications, it is important to ensure their robustness post-deployment. Motivated by controlling the Lipschitz constant of GNN filters with respect to the node attributes, we  propose to constrain the frequency response of the GNN's filter banks. We extend this formulation to the dynamic graph setting using a continuous frequency response constraint, and solve a relaxed variant of the problem via the scenario approach. This allows for the use of the same computationally efficient algorithm on sampled constraints, which provides PAC-style guarantees on the stability of the GNN using results in scenario optimization. We also highlight an important connection between this setup and GNN stability to graph perturbations, and provide experimental results which demonstrate the efficacy and broadness of our approach.
\end{abstract}


\section{Introduction}

    Graph neural networks (GNNs) have proven to be a powerful method for network-based learning tasks, achieving state-of-the-art performance in many applications such as epidemic spread prediction (\cite{kapoor2020examining, epidemicsGNN}), resource allocation (\cite{gao2020resource}), and decentralized control (\cite{Tolstaya2019LearningDC, yang2021communication}). The success of GNNs can be largely attributed to the graph convolution operation, which yields many desirable properties, such as permutation invariance and equivariance (\cite{keriven2019universal}) and transferability (\cite{ruiz2020graphon}) to graphs of varying size.  However, like many other neural network models, GNNs have been shown to be particularly vulnerable to data shifts, perturbations, noise, and many other forms of attacks. This is of critical importance for control-based applications, where input data, such as sensor inputs or infection counts, can be inherently noisy. 
    
    An important property of GNNs that makes a distinction between GNN robustness and traditional neural network robustness is the fact that GNNs are models with two inputs: the graph signals (i.e. node attributes), and the graph adjacency matrix itself. Thus, there are two separate ways in which a GNN can experience shifts on the data; either through (i) shifts of the graph signals they operate on (\cite{CertifiableRobustGCN}), (ii) shifts of the graph adjacency matrix (\cite{bojchevski19a, GraphPoisoning, GNNStability, cervino2021training, dai2018adversarial}), or (iii) both (\cite{advGNNZugner, RGCN}). Many of the aforementioned works enforce robustness to such attacks using variants of learning under distributional shifts (\cite{biggio2013evasion,szegedy,carlini2017,madry2017towards, hendrycks2019natural, Duchi2018LearningMW, Robey2020ModelBasedRD, robey2021adversarial}) and applying them to GNN settings. In these works, one typically assumes some sort of model of how the data might be shifted (e.g. $\ell^\infty$ attacks) and aims to ensure robustness against attacks in line with these specified models. 
    
    In many safety-critical applications that use GNNs, however, it is likely that one does not know the sort of noise to be encountered after the GNN has been deployed. Therefore, it would be useful to have a method that is \textit{agnostic} to the data shift model of the system. In this paper, we take an approach that follows Lipschitz-training methods in robust machine learning (\cite{ParsevalNetworks, fazlyab2019efficient, LipschitzBounds}), which enforce stability and robustness by constraining the Lipschitz constant of the neural network during training, and assume no knowledge of a data shift model. Under this paradigm for GNNs, we demonstrate how both forms of GNN stability (shifts on node attributes and shifts on graph structure) have an inherent connection to the frequency responses of the GNN filters, which are simply polynomials with coefficients as the filter coefficients. Therefore, both forms of robustness for GNNs can be achieved via a constraint on the graph filter frequency response. 
    
    In what follows, we first consider shifts of graph signals on a fixed graph, and motivate a frequency response constraint by demonstrating that the Lipschitz constant of a graph filter (w.r.t. graph signals) is given by the $\ell^\infty$ norm of the frequency response evaluated on the spectrum of the graph shift operator. This finite constrained problem is easily solvable via $\ell^\infty$ projection. To be universally stable to graph signal shifts across a class of graphs, we extend the discrete frequency response constraint to a continuous constraint. We propose a semi-infinite problem formulation to enforce the continuous constraints, which we relax to a chance-constrained problem for computational tractability. The fact that frequency response is given by polynomials on the filter weights provides inherent structure to the problem. Therefore, efficient methods can be used such as the scenario approach, which samples the constraints and provides sample complexity guarantees via VC theory. We then show how GNN stability with respect to shifts on the graph structure, which can be enforced by a constraint on the frequency response's derivative (\cite{GNNStability}), can also be easily performed using our framework. We provide experiments that demonstrate the efficacy of our approach, demonstrating stability to various noise distributions, as well as adversarial attacks, in several application settings.
    
\section{Background}
    We approach graph neural networks from the graph signal processing point of view (\cite{GNN}). Specifically, we consider an $n$-node graph $\calG = (\calV, \calE)$, and use it to define linear, shift-invariant filters that operate on graph signals $x \in \setR^n$, where the $i$-th entry $x_i$ corresponds to the value of the $i$-th node. Let $S \in \setR^{n \times n}$ denote a graph shift operator (e.g. adjacency or Laplacian of $\calG$), which we assume to be symmetric (i.e. $\calG$ is undirected). We can define a $K$-order graph filter as a vector of coefficients $h \in \setR^K$. In order to filter some graph signal $x$ with $h$ with respect to the graph shift operator $S$, we use the graph convolution $h *_S x = \sum_{k=1}^K h_k S^{k-1} x $. Intuitively, the graph convolution takes shifted copies of $x$ and weights them by $h_k$. Here, the ``shift'', which in linear time-invariant filters is simply a time delay, corresponds to aggregations of the $k-1$-th hop neighborhood $S^{k-1}x$.
    
    Defining graph convolutions this way allows one to define the frequency response of filter $h$ as $H(\lambda) = \sum_{k=1}^{K} h_k \lambda^{k-1}$ which is a $K-1$-degree polynomial with coefficients as the filter coefficients, evaluated at some frequency $\lambda \in \setR$. If we let $S = VD V^H$ be the graph Fourier transform (eigendecomposition) of $S$, where $D = \mathrm{diag}(\{\lambda_1, \dots, \lambda_n\})$ contains the spectrum of $S$, then the graph convolution can be written as $h *_S x = V \left(\sum_{k=1}^K h_k D^{k-1} \right) V^H x = V H(D) V^H x$, where $H(D) = \mathrm{diag}(H(\lambda_1), \dots, H(\lambda_n))$ contains the frequency response of $h$ evaluated on the eigenvalues of $S$. The frequency response (and the eigenvalues on which it is evaluated) is a key tool with many implications for GNN stability properties, as will be described in Section~\ref{problem} and \ref{dynamic}.
    
    A graph neural network (GNN) can then be built from graph filters and described as a cascade of $Q$ layers, where each layer is given by a graph filter bank followed by a pointwise nonlinearity (\cite{GNN}). Concretely, at the $q$-th layer, let $\calH^{(q)} \in \setR^{G_{q-1} \times G_q \times K}$ denote the filter tensor, which is simply $G_{q-1}$ filter banks containing $G_q$ filters each, where each filter contains $K$ taps. 
    We define $G_0 = d$ to be the number of input feature dimensions of a sample of our data $X \in \setR^{n \times d}$, which is a graph signal with $d$ features. Then, the output of layer $q$ is simply
    \begin{equation}
    	X_{q} = \sigma\left(\sum_{k=1}^K S^{k-1} X_{q-1} \calH^{(q)}_k\right)
    \end{equation}
    where $\sigma(\cdot)$ is a pointwise nonlinearity, and $X_0 := X$ is the input to the GNN. 
    We denote the final output after all $Q$ layers on input $X$ as $\Phi(X; S, \calH^Q)$ where $\calH^Q = \{\calH^{(1)}, \dots, \calH^{(Q)}\}$ is the collection of all filter tensors across all layers. These are the weights that parametrize the GNN. 

    \subsection{Notions of GNN Stability}
    We now formalize the two notions of GNN stability corresponding to perturbations on the two inputs of a GNN: the underlying graph and the signal supported upon it. The first notion is provided in (\cite{GNNStability}).
    
    \begin{definition}[GNN Stability to Graph Perturbations] \label{def:graph_perturbations}
        A GNN $\Phi$ is $C_1$-\textit{stable to graph perturbations} with respect to a set of node attributes $\calX$ if 
        \begin{equation}
            \sup_{x \in \calX} \|\Phi(x; S, \calH) - \Phi(x; S', \calH)\| \leq C_1 d(S, S') \quad \forall S, S' \in \setS
        \end{equation} for distance $d$ on graph shift operators.
    \end{definition}
    \begin{definition}[GNN Stability to Signal Perturbations] \label{def:signal_perturbations}
         A GNN $\Phi$ is $C_2$-\textit{stable to signal perturbations} with respect to a set of graphs $\setS$ if 
        \begin{equation}
            \sup_{S \in \setS}||\Phi(x;S,\calH)-\Phi(x';S,\calH)||\leq C_2||x-x'||\quad\forall x,x'\in\calX
        \end{equation}
        
    \end{definition}
    
    While (\cite{GNNStability}) provides conditions and bounds pertaining to the former, the latter has not been explored in depth. Moreover, the latter of these notions is of particular importance in control settings where knowledge of the graph (e.g. communication network, proximity, etc.) is well known, but graph signals often originate from sensors that may be noisy and/or miscalibrated. In Section~\ref{problem} and \ref{dynamic}, we will first focus on the stability notion in Definition~\ref{def:signal_perturbations}, and then connect back to Definition~\ref{def:graph_perturbations} in Section~\ref{connection}.
    
\section{Problem Formulation}\label{problem}
    In this section, we motivate our approach, which is to constrain the frequency response of the GNN filters during training. We first discuss the case in which there is a fixed graph. Later, in Section \ref{dynamic}, we generalize to the case where we may have set of graphs, which are applicable in time-varying problem settings.

    \subsection{Lipschitz Filters in the Graph Signal Domain} \label{lipschitz}
    
       If we want to impose the property that signals which are close in some distance will result in similar graph filter outputs, i.e. that Definition~\ref{def:signal_perturbations} holds, we can try to restrict the Lipschitz constant of the filter, which is defined as 
        \begin{equation}
            \Lip_p(h) \triangleq \sup_{x_1 \neq x_2 \in \setR^n}\frac{\|h *_S x_1 - h *_S x_2\|_p}{\|x_1 - x_2\|_p}
        \end{equation}
        which is defined with respect to the $p$-norm, $\|x\|_p = (\sum_{i=1}^n |x_i|^p)^{1/p}$. The following lemma establishes the connection between the Lipschitz constant of $h$ and the frequency response of $h$, $H(\lambda)$.
        
        \begin{lemma} \label{lemma:lipschitz}
            Let $h \in \setR^K$ be a vector of filter coefficients of length $K$. The Lipschitz constant of $h$, which is taken with respect to changes in input graph signals on some graph shift operator $S$, is given by $\Lip_p(h) = \max_{\lambda \in \Lambda(S)} |H(\lambda)|$, where $\Lambda(S)$ is the spectrum of $S$.
        \end{lemma}
        
        \begin{proof}
        Recall that $h *_S x = V H(D) V^H$, where $V$ are the eigenvectors of $S$, and $H(D)$ is a diagonal matrix containing the frequency response of $h$ evaluated at $\Lambda(S):=\{\lambda_1,\dots,\lambda_n\}$, which is the spectrum of $S$. Then
        \begin{align}
            \|h *_S x_1 - h *_S x_2 \|_p &= \|VH(D) V^H (x_1-x_2)\|_p 
            = \|H(D)V^H (x_1-x_2)\|_p \nonumber\\
            &\leq \|H(D) V^H \|_p \|x_1 - x_2\|_p 
            = \|H(D)\|_p \|x_1-x_2\|_p
        \end{align}
        Since this upper bound on $\|h *_S x_1 - h *_S x_2 \|_p$ can be achieved when $x_1 - x_2$ is equal to the vector achieving the max in the induced $p$-norm of $H(D)V^H$, we have that $\Lip_p(h) = \|H(D)\|_p = \max_{i \in \{1,\dots,n\}} |H(\lambda_i)|$, where the second equality holds since $H(D)$ is diagonal.
        \end{proof}
        
         Thus, to enforce stability to graph signal perturbations, our objective is to \textit{constrain the maximum absolute value of $H(\lambda)$}, where the max is taken over the eigenvalues of the graph shift operator $S$. When this is applied to each layer of a GNN, we arrive at the following statistical risk minimization problem, where $(X,y) \sim \calD$ is the data distribution and $\ell$ is some loss function:
        \begin{align}
            \minimize_{\calH^{(1)},\dots,\calH^{(Q)}}\qquad&\E_{(X,y)\sim\calD}[\ell(\Phi(X;S,\calH^Q),y)] \nonumber \\
            \subjectto\qquad& |H^{(q)}_{f,g}(\lambda)|\leq c\quad \forall q\in[Q],f\in[G_{q-1}],g\in[G_q], \forall \lambda\in\Lambda(S)
        \end{align}
        where $\Lambda(S)$ refers to the spectrum of matrix $S$, and $H^{(q)}_{f,g}(\lambda)$ is the frequency response of filter $g$ in the $f$-th bank of the $q$-th layer. We can simplify notation by defining 
        \begin{equation}
            H^*(\lambda)\triangleq \max_{q\in[Q]} \max_{f\in[G_{q-1}]} \max_{g\in[G_q]}|H^{(q)}_{f,g}(\lambda)|
            \label{eq:max_freq_resp}
        \end{equation} 
        which is simply the max absolute frequency response of all filters in the GNN evaluated at $\lambda$.
        This leads us to the following constrained learning problem:
        \begin{align}
            \minimize_{\calH^{(1)},\dots,\calH^{(Q)}}\qquad&\E_{(X,y)\sim\calD}[\ell(\Phi(X;S,\calH^Q),y)] \nonumber\\
            \subjectto\qquad&H^*(\lambda)\leq c\quad \forall \lambda\in\Lambda(S)
            \label{eq:base_prob}
        \end{align}
        
        \begin{remark}[Multiplicative Lipschitz Constant of the GNN]
        Solving \eqref{eq:base_prob} allows one to guarantee a bound on the Lipschitz constant that is multiplicative in the number of layers of the GNN with respect to the input graph signals. This allows one to ensure stability in the sense of Def.~\ref{def:signal_perturbations} on a singleton graph, i.e. $\setS = \{S\}$. Sec.~\ref{dynamic} will generalize this to general graph sets $\setS$. Also, the constraint $c$ can also differ across layers depending on the problem and desired outcome. 
        \end{remark}

        \subsection{Problem Realization for Static Graphs}\label{static}
        To solve $\eqref{eq:base_prob}$, we would like to write the constraint directly in terms of the GNN filter weights. To do so, let $\mathcal{V}_{\Lambda(S)}$ be the Vandermonde matrix evaluated on the values of $\Lambda(S)$, and truncated to the length of the filters (we assume $K < |\Lambda(S)|$), i.e. 
        \begin{equation}
            \mathcal{V}_{\Lambda(S)} \triangleq \begin{bmatrix} 
            1 & \lambda_1 & \lambda_1^2 & \dots & \lambda_1^{K-1} \\
            1 & \lambda_2 & \lambda_2^2 & \dots & \lambda_2^{K-1} \\
            \vdots & \vdots & \vdots & \ddots & \vdots \\
            1 & \lambda_m & \lambda_m^2 & \dots & \lambda_m^{K-1}
            \end{bmatrix}
        \end{equation}
        The Lipschitz constant can be expressed directly in terms of the filter coefficients via the Vandermonde matrix: $\Lip_p(h) = \max_{i\in\{1,\dots,n\}} \left|\sum_{k=1}^{K} h_k \lambda_i^{k-1}\right|= \|\mathcal{V}_{\Lambda(S)} h \|_\infty$. Hence we simply need to constrain $\|\mathcal{V}_{\Lambda(S)} h \|_\infty \leq L$ in order to ensure $L$-Lipschitzness of filter $h$ with respect to the inputs.
        
            
            

        This implies that in order to solve \eqref{eq:base_prob}, we should constrain the the $\infty$-norm of $\calV_{\Lambda(S)}h$ for each filter $h$ in the GNN, and solve the following problem for a $Q$-layer GNN:
        \begin{align}
                \minimize_{\mathcal{H}^{(1)},\dots,\mathcal{H}^{(Q)}}\quad& \mathop{\mathbb{E}}_{(x,y)\sim \mathcal{D}} [\ell (\Phi(x; S,\mathcal{H}^Q), y)] \nonumber \\ 
                \mathrm{subject~to}\quad 
                &\|\mathcal{V}_{\Lambda(S)}\mathcal{H}^{(q)}_{(f,g,:)}\|_\infty \leq c \quad \forall q \in [Q], f \in [G_{q-1}], g \in [G_q]
                \label{eq:xstable}
        \end{align}
        where $\mathcal{H}^{(q)}_{(f,g,:)} \in \mathbb{R}^{K}$ is filter $g$ in the $f$-th bank in the $q$-th layer. This is now an optimization over the weight tensors with $\sum_{q=1}^Q G_q G_{q-1}$ constraints ($G_0$ is the number of input channels of $X$, and layer $q$ has $G_{q-1}$ inputs and $G_q$ outputs). Assuming that the nonlinearities are $1$-Lipschitz, the above optimization problem guarantees a GNN Lipschitz constant bound, described in the previous section. 
        
    \subsection{Enforcing Frequency Response Constraints on a Finite Set of Eigenvalues}
        In order to solve \eqref{eq:xstable}, we would like to use projected gradient descent, which guarantees that the filters we learn satisfy the constraints. While primal-dual algorithms have been used before, e.g. in (\cite{cervino2021training}), they do not guarantee the solution lies in the feasible set, due to the lack of strong duality. To solve the projection 
        \begin{equation}
            \mathrm{Proj}_{\{h:\|\calV_{\Lambda(S)}h\|_\infty \leq c\}} (g) = \argmin_{h:\norm{\calV_{\Lambda(S)}h}_\infty \leq c} \norm{h-g}_2
            \label{eq:original_obj}
        \end{equation}
        we may first change the basis to $V$ where $\calV_{\Lambda(S)}=U\Sigma V^\top$ is the SVD of $\calV_{\Lambda(S)}$, and $\Sigma_{ii} = \sigma_i$, for $i \leq K$. Since the set $\{h \in \setR^K: \|\Sigma h\|_\infty \leq c\} = \{h \in \setR^K: |h_i| \leq \frac{c}{|\sigma_i|}, i \in [K]\}$ is a box, the solution in the original basis \eqref{eq:original_obj} is given by $V h^{\mathrm{proj}}$, where
        \begin{equation}
            h^{\mathrm{proj}}_i = [\mathrm{Proj}_{\{h:\|\Sigma h\|_\infty \leq c\}}(V^\top g)]_i = \begin{cases} \mathrm{sign}([V^\top g]_i) \frac{c}{|\sigma_i|} \qquad |[V^\top g]_i| > \frac{c}{|\sigma_i|}\\ [V^\top g]_i \qquad \qquad \qquad |[V^\top g]_i| \leq \frac{c}{|\sigma_i|}  \end{cases}
            \label{eq:projection}
        \end{equation}
          The procedure in \eqref{eq:projection} is easily tensorized for each weight tensor $\calH^{(q)}$, since $\calH^{(q)}$ is just $G_{q-1} \times G_q$ filter vectors of length $K$. This yields the algorithm described in Algorithm~\ref{alg:parseval}.

        \begin{algorithm}[t]
            \DontPrintSemicolon
            \caption{Discrete frequency response constraints via projected SGD}
            \label{alg:parseval}
            \KwIn{Set of eigenvalues $\Lambda$, step size $\eta_t$, batch size $B$}
            \KwOut{Trained weight tensors $\mathcal{H}^Q \in \mathbb{R}^{Q \times G_{q-1} \times K \times G_q}$}
            \SetKwFunction{GNNLipschitzTraining}{GNNLipschitzTraining}
            \SetKwProg{myproc}{Procedure}{}{}
            Randomly initialize GNN filter weights $\mathcal{H}^Q \in \mathbb{R}^{Q \times G_{q-1} \times K \times G_q}$ \\
            Generate Vandermonde matrix $\calV_\Lambda$ evaluated on $\Lambda$ \\
            Compute the SVD $\mathcal{V}_\Lambda = U\Sigma V^\top$, where $\Sigma_{ii} = \sigma_{i}$\\
            \While{not converged}{
                    Sample a batch $\{(X_i,y_i)\}_{i=1}^B \sim \mathcal{D}$\\
                    $\mathcal{H}^Q \leftarrow \mathcal{H}^Q - \eta_t  \frac{1}{B}\sum_{i=1}^B  \nabla_{\mathcal{H}^Q} \ell(\Phi(X_i;S,\mathcal{H}^Q), y_i)$ \\
                    \For {$q=1,\dots,Q$} { 
                        \For {each filter $h$ in $\calH^{(q)}$} {
                            Solve $h^{\mathrm{proj}} = \mathrm{Proj}_{\{h:\|\Sigma h\|_\infty \leq c\}}(V^\top h)$ using \eqref{eq:projection} \\
                            $h \leftarrow V h^{\mathrm{proj}}$
                        }
                    }
                }
                
        \Return{$\calH^Q$}
        \end{algorithm}

        While \eqref{eq:xstable} was formulated in the setting of learning problems involving only a single static graph $S$, we will see in Section~\ref{dynamic} that this setup extends to the dynamic graph setting as well as methods enforcing GNN stability to graph perturbations, which we connect in Section~\ref{connection}. Therefore, Algorithm~\ref{alg:parseval} can be used as the workhorse of many Lipschitz-based learning methods for GNNs.

\section{Extension to Dynamic Graphs}\label{dynamic}
    In control applications, it is common for the underlying network of interest to be changing in time. This is especially prevalent in time-series applications where the graph changes depending node locality at each time step. For instance, in decentralized control settings, agents may only share information with local neighbors on a dynamic graph (see Section \ref{flocking}). Thus, it is important that networks be stable to inputs on a broad set of possible graphs that it may encounter.

    \subsection{Formulating a Semi-Infinite Optimization Problem}
        More concretely, we wish now to extend the bound outlined in Section \ref{lipschitz}, which was with respect to a single graph, to the broader sense of stability defined in Definition~\ref{def:signal_perturbations}, which is universally stable on some set of graphs $\setS$. Under that definition, it follows that to guarantee input stability on a (potentially infinite) set of graphs $\setS$, the same constraint must be applied to all $\lambda\in \Lambda(\setS) \triangleq \bigcup_{S\in\setS}\{\lambda\in \Lambda(S)\}$. Note that $\Lambda(\setS)$ is now a set function containing the union of eigenvalues of all graphs in $\setS$. The meaning of $\Lambda(\cdot)$ should be understood depending on its argument.

        To consider all graph shift operators $S \in \setS$, it might be the case that the set of eigenvalues $\Lambda(\setS)$ might be very large, which might make \eqref{eq:xstable_infinite} difficult or intractable; however, one can easily obtain simple bounds on $\Lambda(\setS)$. Indeed, by the Gershgorin circle theorem (\cite{GCT}), we have a guarantee that $\Lambda(\setS) \subseteq [-n+1, n]$ for any $\setS \subseteq \{0,1\}^{n \times n}$, i.e. $\setS$ contains graph shift operators that represent adjacency matrices. In later sections, we will see that in real-world settings, we can constrain $\Lambda(\setS)$ to much smaller intervals than the one given via Gershgorin's theorem.
        
        To enforce input stability on a GNN, we can solve the semi-infinite constrained problem:
        \begin{align}
                \min_{\mathcal{H}_1,\dots,\mathcal{H}_Q}\quad& \mathop{\mathbb{E}}_{(X,y)\sim \mathcal{D}} [\ell (\Phi(X; S,\mathcal{H}^Q), y)] \nonumber \\ 
                \mathrm{s.t.}\qquad 
                & H^*(\lambda) \leq c \quad \forall \lambda \in \Lambda(\setS)
                \label{eq:xstable_infinite}
        \end{align}
         where $H^* (\lambda)$ is as defined in Section \ref{problem} and $\Lambda(\setS)\subseteq[-n+1,n]$.
         
         Note that, due to the reduction of Lipschitzness to a condition on the eigenvalues of the graph shift operator, we are able to easily extend our formulation to dynamic graph settings. Constraining our problem over eignevalues rather than families of graphs significantly reduces the dimensionality and computational difficulty of ensuring stability. This is also particularly salient in decentralized control applications where the underlying network can change dramatically, but eigenvalues can be reasonably bounded.

    \subsection{Scenario Optimization}
        Rather than constraining the worst-case eigenvalues, we follow the scenario approach introduced in \cite{Calafiore_Campi_2004}. By relaxing the semi-infinite constraint to a chance constraint which is then sampled, we identify a solution that, with high probability, satisfies our constraints while maintaining computational tractability and sacrificing less performance.
        
        Formally, we introduce a random variable $\lambda\in\Lambda(\setS)\subseteq\setR$ defined on a probability space $(\Lambda(\setS),\calF,\P)$ and relax the semi-infinite constrained problem to a chance constrained problem (CCP):
        \begin{align}
            \minimize_{\mathcal{H}_1,\dots,\mathcal{H}_Q}\quad& \mathop{\mathbb{E}}_{(X,y)\sim \mathcal{D}} [\ell (\Phi(X; S,\mathcal{H}^Q), y)] \nonumber \\
            \subjectto\quad & \P\left(\{\lambda\in\Lambda(\setS)\bigr|\ |H^{(q)}_{(f,g,:)}(\lambda)| \leq c\}\right)\geq1-\epsilon \quad\forall q\in[Q], f\in[G_{q-1}], g\in[G_q] 
            \label{eq:CCP}
        \end{align}
        We then draw a sample of $m$ eigenvalues $\bar\sigma:=\{\bar\lambda_1,\dots,\bar\lambda_m\}$ according to $\P$ on which to enforce the constraint to arrive at the following scenario program:
        \begin{align}
            \minimize_{\mathcal{H}_1,\dots,\mathcal{H}_Q}\quad& \mathop{\mathbb{E}}_{(X,y)\sim \mathcal{D}} [\ell (\Phi(X; S,\mathcal{H}^Q), y)] \nonumber \\
            \subjectto\quad & |H^{(q)}_{(f,g,:)}(\lambda)| \leq c \quad\forall q\in[Q], f\in[G_{q-1}], g\in[G_q],\lambda\in\bar\sigma 
            \label{eq:SP}
        \end{align}
        The following proposition, given by classical VC theory (\cite{COLT}), provides a sample complexity guarantee on the generalization of the scenario approach to solve \eqref{eq:CCP}.
        \begin{proposition}[Sample Complexity Bound via VC Theory]\\
        For any $\delta,\epsilon\in(0,1)$ and $\bar\sigma$ drawn according to $\P^{\otimes m}$ such that 
        \begin{equation}
            m\geq\Bigr\lceil\frac{4}{\epsilon}\left(K\ln{\left(\frac{12}{\epsilon}\right)}+\ln{\left(\frac{2}{\delta}\right)}\right)\Bigr\rceil
        \end{equation}
        the solution of \eqref{eq:SP} satisfies $\P\left(\{\lambda\in\Lambda(\setS)\bigr|\ H^{(q)}_{(f,g,:)}(\lambda) \leq c\}\right)\geq1-\epsilon$ with probability at least $1-\delta$ for all $q\in[Q], f\in[G_{q-1}], g\in[G_q] $.
        \end{proposition}
        \begin{proof}
        First note that the constraint on the absolute value of $H(\lambda)$ can be broken into two constraints, each on a polynomial of $\lambda$ of the same degree as in the original constraint. Next, observe that each of the $2\sum_{q=1}^Q G_{q-1}G_q$ constraints is described by a $(K-1)$-degree polynomial in $\lambda$. Thus, the family of functions describing the constraint has VC dimension $K$. The result then follows from theorem 8.4.1 in (\cite{COLT}) .
        \end{proof}
        
        This result implies that in order to solve \eqref{eq:CCP}, we can simply sample a (large enough) number of eigenvalues according to $\P$, enforce the frequency response constraints on those eigenvalues, and guarantee that there is at least a $1-\epsilon$ fraction of constraints in $\Lambda(\setS)$ that are satisfied, with high probability. Thus, since this procedure has a finite number of constraints, we can again use Algorithm~\ref{alg:parseval} to solve problems in the dynamic graph setting, where the set of eigenvalues $\Lambda$ is no longer those contained in the spectrum of a particular graph shift operator, but rather the random draw of eigenvalues $\bar{\sigma}$. An appealing property of this approach is that since our Lipschitz constraints have reduced to enforcing a constraint on a single-variable polynomial (the frequency response), the sample complexity is linear in the number of filter taps, which is generally small. Note that one need not know the distribution $\P$ to follow this procedure. Simply being able to sample eigenvalues on which to enforce constraints is sufficient and can be done using the networks in training data.
        
    
\section{Connections to GNN Stability Under Graph Shifts}      \label{connection}
    The ability to enforce Lipschitzness of a GNN through conditions on its frequency response on specific eigenvalues also engenders connections between the various notions of stability. In (\cite{GNNStability}), the authors investigate conditions on which a GNN is stable with respect to graph shifts, i.e. the stability notion given in Definition~\ref{def:graph_perturbations}. They show that one can guarantee stability under relative graph shifts if the network's filters are \textit{integral Lipschitz} and the distance on graph shift operators $d$ is taken to be operator distance modulo permutations, i.e. $d(S, S') = \min_{P \in \calP} \|SP^\top - P^\top S'\|$, where $\calP$ is the set of permutation matrices.
    
    The integral Lipschitz condition implies that one simply needs to ensure that the function $\lambda \mapsto \lambda \frac{dH(\lambda)}{d\lambda}$ is bounded. Since $H(\lambda)$ is a $K-1$ degree polynomial, $\lambda \frac{dH(\lambda)}{d\lambda} = \sum_{k=1}^K h_k (k-1)\lambda^{k-1}$ is also a $K-1$ degree polynomial. If we define $h' = [0, h_1, 2h_2,\dots, (K-1)h_{K-1}]^\top$, we can use the exact same setup in the previous two sections to enforce this constraint on the modified filter coefficients $h'$. Moreover, this connection implies that if one enforces stability to graph shifts, one can obtain a bound on the stability under graph signal shifts, and vice versa.
    
\section{Experimental Results}
    \subsection{Static Graphs}
    In the static graph setting, we evaluate source localization, where there is a static graph shift operator $S$, and some information is allowed to propagate along the graph. Given the time series of propagated information, the GNN's task is to predict the original information source at a previous time step. We train a GNN with no constraints, and one with frequency response constraints via Algorithm~\ref{alg:parseval} (`Lipschitz'). To evaluate these models, we apply two types of noise at test time: Gaussian noise, and adversarial noise. For the former, if $X, y$ is our input graph signal and target, then we input $X^{\mathrm{AWGN}}_\sigma = X + \calN(0, \sigma I)$ to the GNN. For the latter, we use $X^{\mathrm{adv}}_\epsilon = \argmax_{X':\|X'-X\|_\infty \leq \epsilon} \ell(\Phi(X';S,\calH^Q), y)$. For both noises, we also compare with a GNN trained against each noise model (i.e. AWGN data augmentation and $\ell^\infty$-PGD training). 
    
    Shown in Figure~\ref{fig:sourceloc}, the performance of the standard GNN degrades quickly with increasing noise on the input signal. Data augmentation is able to slow this degradation and improve performance overall, but it too struggles (particularly in the AWGN case) once the power of the noise applied during evaluation surpasses that of its noisy training data. The Lipschitz GNN, however, maintains high performance with increasing noise while sacrificing very little performance on clean data. In the adversarial case, our method performs as well as the defense designed against the adversary, without any knowledge of the adversary. This corroborates the analytical results and shows that frequency response constraints are an effective method of ensuring robustness to input signal noise while being \textit{agnostic} to the data shift model.

    \begin{figure}[t]
        \centering
        {%
            \subfigure{%
            \label{fig:awgn}
            \includegraphics[scale=0.45]{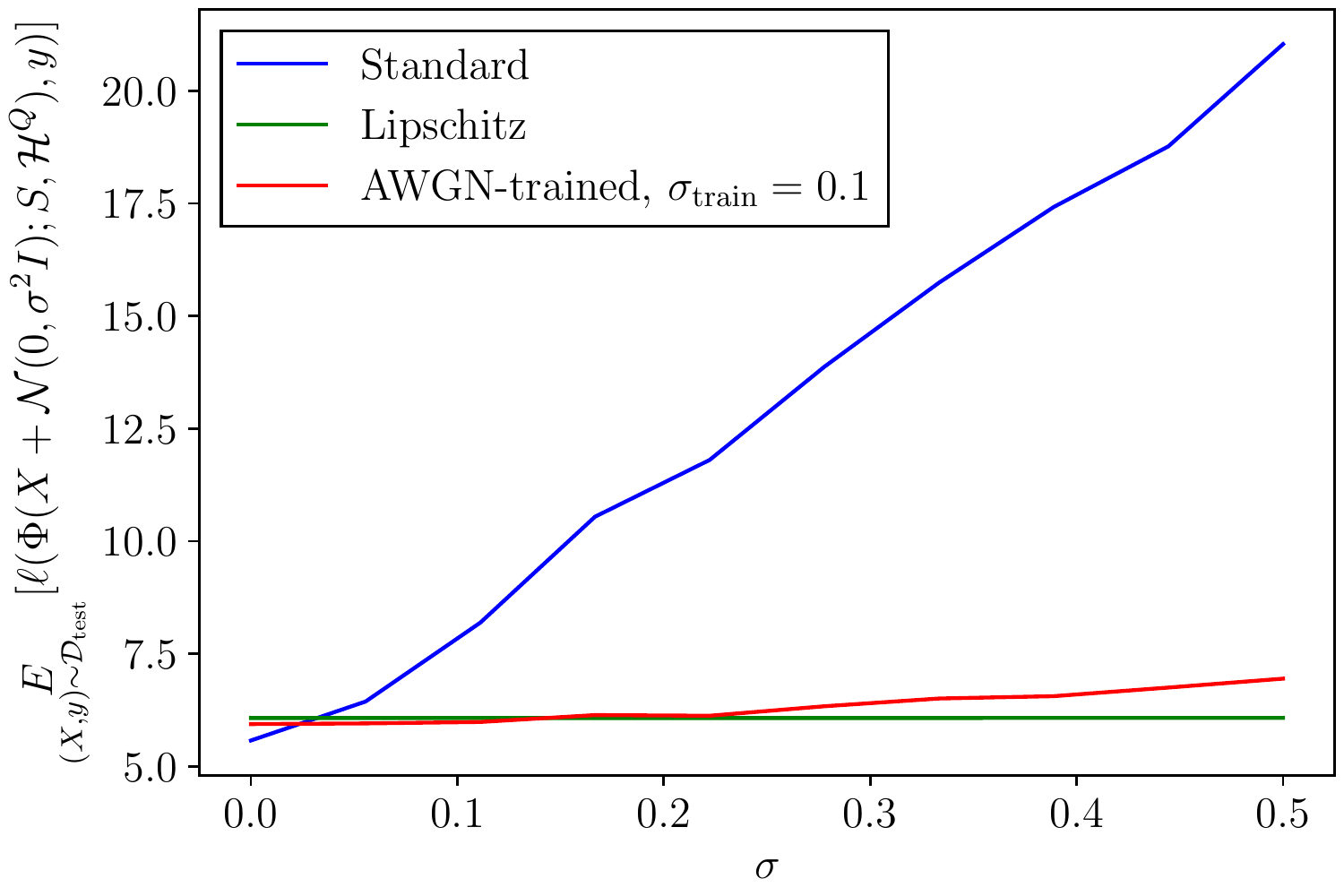}
        }\qquad 
            \subfigure{%
            \label{fig:unif}
            \includegraphics[scale=0.45]{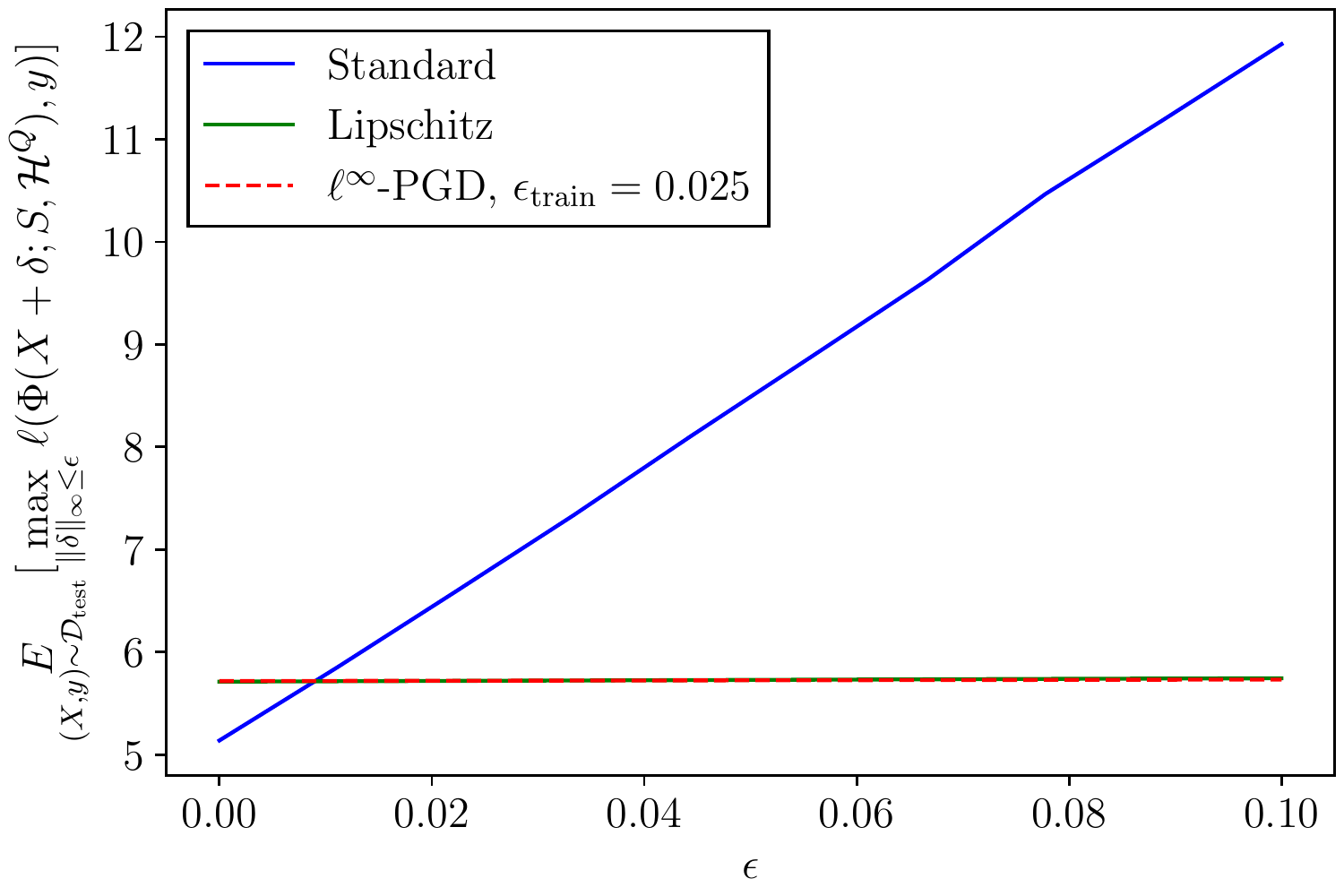}
            }
        }
       \caption{Source localization. In (\textit{a}), AWGN. In (\textit{b}), $\ell^\infty$ adversary.}
        \label{fig:sourceloc}
    \end{figure}

    \subsection{Dynamic Graphs}\label{flocking}
    In the dynamic graph setting, we use an example involving agents flocking together in a decentralized manner (\cite{gama2020graph}). In this example, there are agents that seek to move in the same direction at some velocity without hitting each other. If each agent is aware of all the other agents' positions and velocities at each time step, then each agent can apply the optimal \textit{centralized} policy. However, in practice, these agents are constrained by communication, and cannot ascertain information of far-away agents instantaneously. The objective is to learn a decentralized, communication-constrained policy that mimics the optimal centralized policy, using a GNN that respects the communication constraints. Due to the movement of the agents, the graphs change at each time step. Furthermore, the agents' sensor inputs of neighboring positions and velocities may be noisy, which we hope to combat using our stability framework over dynamic graphs (Section~\ref{dynamic}).
   
    In order to apply the scenario approach, we need to ascertain $\Lambda(\setS)$, where $\setS$ contains all communication-constrained graphs of the flocking agents. In practice, we set $\Lambda(\setS) = [a, b]$, where $a$ and $b$ are the min and max of the all the eigenvalues of all the graph shift operators in the training set, and sample $m$ constraints according to $\P = \mathrm{Unif}([a,b])$. In the flocking example, we set $[a,b]=[-0.75, 1.25]$, and $m=1000$. We enforce the constraints on these points using Algorithm~\ref{alg:parseval}, \begin{wrapfigure}{r}{0.5\textwidth}
        \centering
        \includegraphics[width=0.4\textwidth]{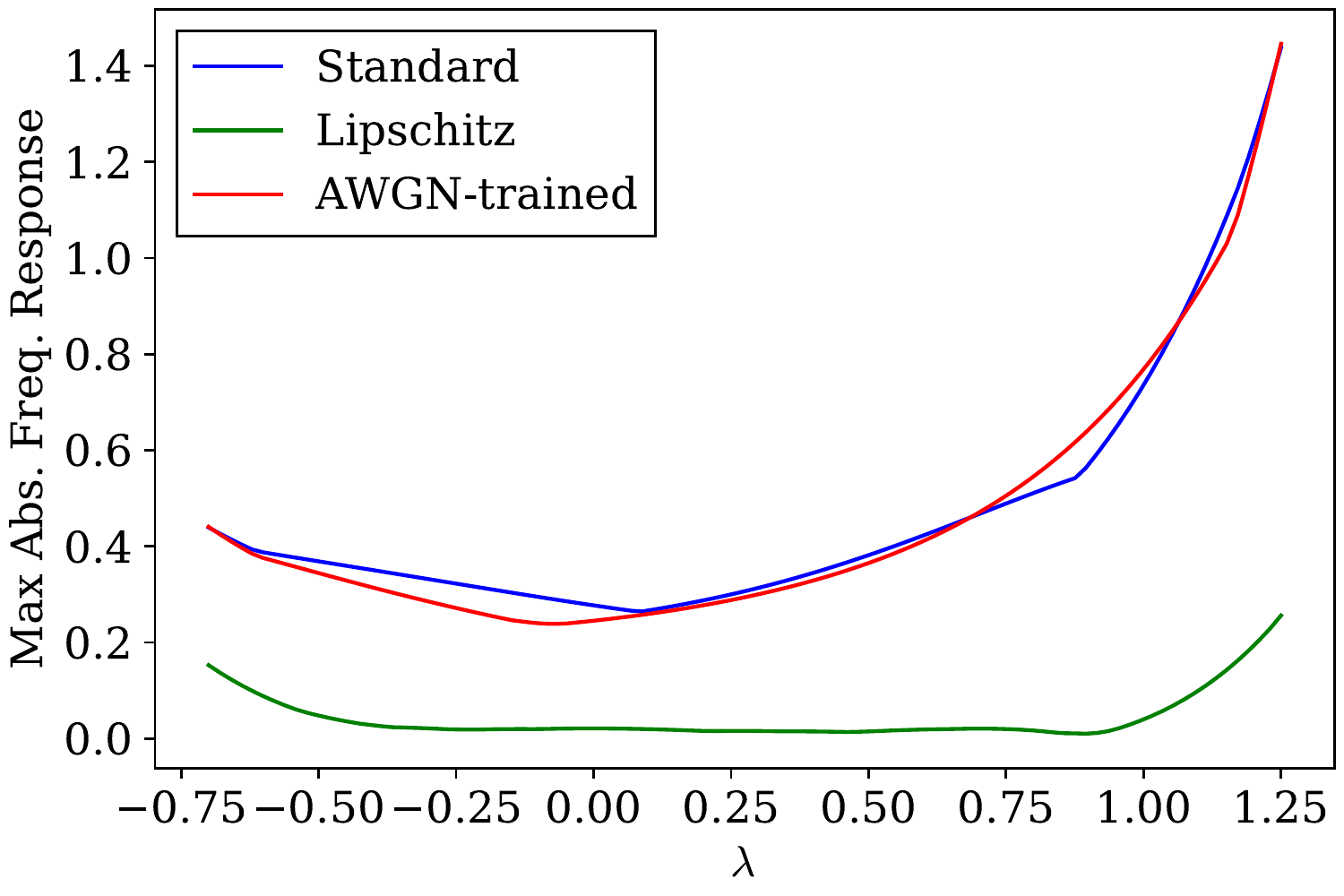}
        \vspace{-1em}
        \caption{Max absolute frequency response $H^*(\lambda)$ \eqref{eq:max_freq_resp} over filters of GNNs.}
        \vspace{-1em}
        \label{fig:freq_resp}
    \end{wrapfigure} and evaluate on the same setup as in the previous section, where both types of noise are added to the GNN controller's inputs for each agent. As shown in Figure~\ref{fig:flocking}, we see that the Lipschitz constrained GNN is again able to maintain stability in a model-agnostic fashion and outperform AWGN data augmentation when evaluated on noise that differs from its training set, and perform nearly as well as $\ell^\infty$-PGD when evaluated on $\ell^\infty$ adversarial attacks. In Figure~\ref{fig:freq_resp}, enforcing the frequency response constraints to each of the filters in the GNN via scenario approach does generalize to the constraints over the continuous range $[a,b]$ in practice, while data augmentation does not provide any additional constraints.

    \begin{figure}[t]
        \centering
        {%
            \subfigure{%
            \label{fig:awgn2}
            \includegraphics[scale=0.45]{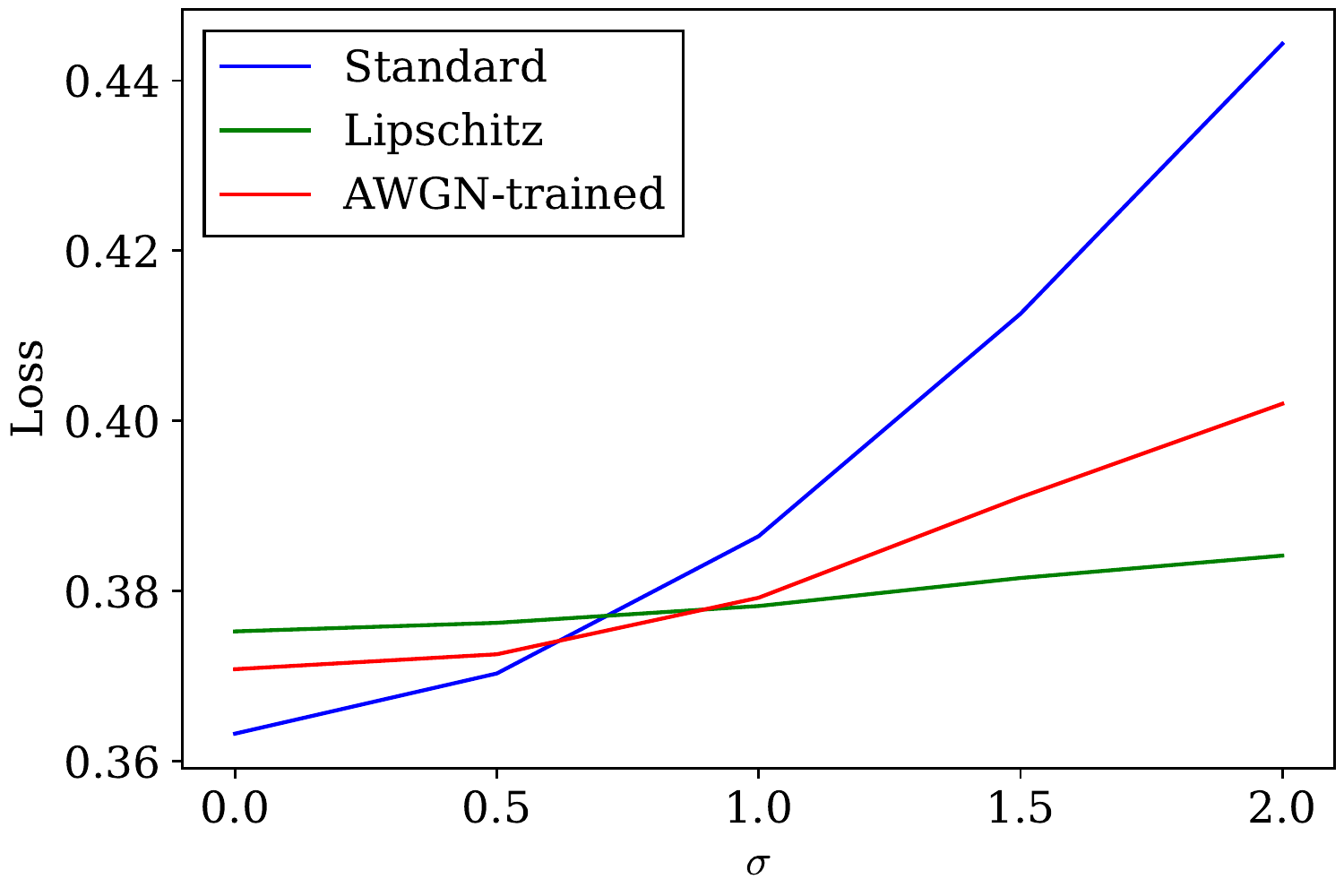}
        }\qquad 
            \subfigure{%
            \label{fig:unif2}
            \includegraphics[scale=0.45]{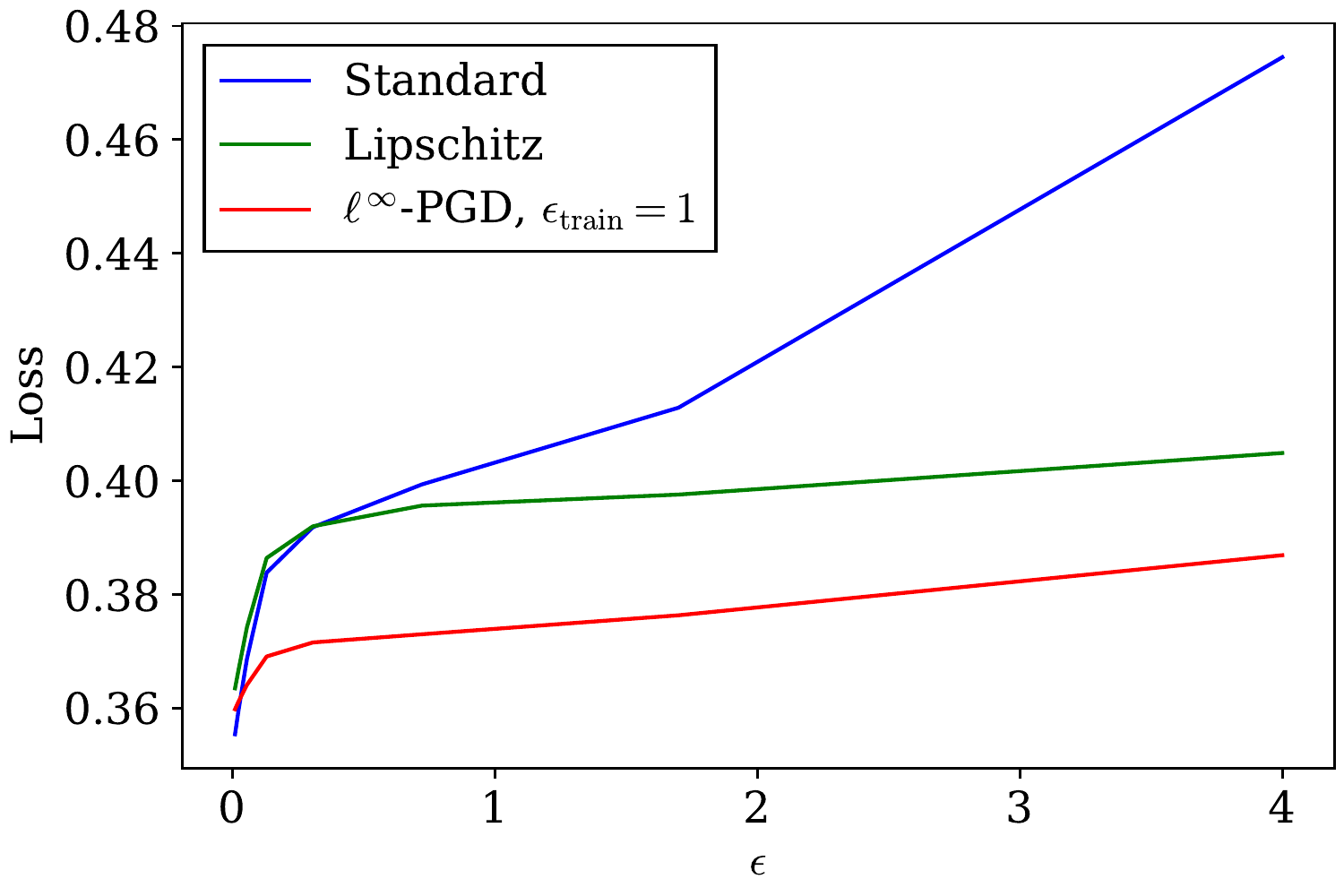}
            }
        }
        \caption{Decentralized control of flocking agents. Loss is measured as the error between the GNN's policy and the optimal flocking policy, averaged over the horizon. In (\textit{a}), evaluate on AWGN-perturbed data. In (\textit{b}),  $\ell^\infty$ adversarial noise. }
        \label{fig:flocking}
        \vspace{-1em}
    \end{figure}

\section{Conclusion}
In this paper, we propose a simple constrained optimization framework for enforcing stability to GNNs by controlling the Lipschitz constants. We show that this framework encompasses several different notions of GNN stability, and how scenario optimization allows for efficient computation with reasonable PAC-style guarantees. Experiments on noisy networked control settings, one in source localization, and one in decentralized flocking of agents, demonstrate the efficacy of our approach. 

\acks{The work of Eric Lei is supported by a NSF Graduate Research Fellowship. The work of Raghu Arghal and Shirin Saeedi Bidokhti is supported by NSF CAREER Award 2047482 and NSF Grant 1910594. We thank Hamed Hassani, Alejandro Ribeiro, and Alexander Robey for helpful discussions.}

\bibliography{ref}



\end{document}